\documentclass{article} 
\usepackage{iclr2027_conference,times}

\usepackage[hyphens]{url}  
\usepackage{graphicx} 
\usepackage{natbib}  
\usepackage{caption} 
\usepackage{algorithm}
\usepackage{algorithmic}

\usepackage{siunitx}

\usepackage{wrapfig}

\usepackage{mathtools}
\usepackage{amsthm}
\usepackage{amsfonts}

\newtheorem{theorem}{Theorem}
\newtheorem{lemma}[theorem]{Lemma}   

\newtheorem{proposition}[theorem]{Proposition}
\theoremstyle{remark}

\usepackage{makecell}
\usepackage{newfloat}
\usepackage{listings}
\DeclareCaptionStyle{ruled}{labelfont=normalfont,labelsep=colon,strut=off} 
\floatstyle{ruled}
\newfloat{listing}{tb}{lst}{}
\floatname{listing}{Listing}

\usepackage{booktabs}
\usepackage[table]{xcolor}

\usepackage{subcaption}

\usepackage{amsmath}

\usepackage[most]{tcolorbox}

\usepackage[normalem]{ulem}
\usepackage{cancel}

\newcommand{\striplabel}[4][inside]{%
  \begin{tikzpicture}
    \node[inner sep=0] (img) {\includegraphics[width=\linewidth]{#2}};
    \ifthenelse{\equal{#1}{above}}{%
      \node[anchor=south west, xshift=5pt, yshift=2pt,
            fill=#3, rounded corners=2pt, inner sep=4pt,
            font=\small\bfseries] at (img.north west) {#4};%
    }{%
      \node[anchor=north west, xshift=5pt, yshift=-5pt,
            fill=#3, rounded corners=2pt, inner sep=4pt,
            font=\small\bfseries] at (img.north west) {#4};%
    }
  \end{tikzpicture}}
\usepackage{newfloat}
\usepackage{listings}
\DeclareCaptionStyle{ruled}{labelfont=normalfont,labelsep=colon,strut=off} 
\floatstyle{ruled}
\newfloat{listing}{tb}{lst}{}
\floatname{listing}{Listing}

\usepackage{amsmath,amsfonts,bm}

\def\eqref#1{equation~\ref{#1}}

\def\1{\bm{1}}

\DeclareMathAlphabet{\mathsfit}{\encodingdefault}{\sfdefault}{m}{sl}
\SetMathAlphabet{\mathsfit}{bold}{\encodingdefault}{\sfdefault}{bx}{n}

\usepackage{hyperref}
\usepackage{url}

\title{Bellman Meets Lyapunov: Unsupervised\\Reinforcement Learning via Mastering Chaos}

\author{Tristan Shah$^{1*}$, Wooyoung Chung$^{1*}$, Volodomyr Makarenko$^{1}$, Juan Wachs$^{2}$, Stas Tiomkin$^{1\dagger}$ \\
$^{1}$Texas Tech University \enspace $^{2}$Purdue University \enspace $^{*}$Equal contribution \enspace $^{\dagger}$stas.tiomkin@ttu.edu
}

\iclrfinalcopy 
\begin{document}

\maketitle
\lhead{Preprint. Under review.}

\begin{abstract}
Reinforcement learning (RL) is a powerful paradigm for training agents, yet its success rests on domain expertise of human engineers who design informative reward signals for every new task. Unsupervised RL aims to reduce this engineering with intrinsic motivation (IM): reward signals that emerge from the agent environment interaction itself.
Existing IM objectives, however, involve the selection of information variables, which re-introduces domain expertise the field has sought to eliminate.
We introduce \textbf{Forward CIP} (F-CIP), an RL-native formulation of the Controllable Information Production (CIP) objective, which is defined by the system's dynamics alone and requires no such selection. We prove that F-CIP is compatible with RL and demonstrate its effectiveness with existing algorithms. Training agents with F-CIP results in unsupervised discovery of primitive behaviors such as balancing and maintaining controllability, which are essential for more complex robot behaviors. Paired with a simple forward-velocity reward, our method produces coordinated gaits such as hopping and running which otherwise require reward engineering to learn. Our website and code are available at the following link: \href{https://anonymous.4open.science/w/cip-rl-page-48C5}{\color{magenta}\nolinkurl{anonymous.4open.science/w/cip-rl-page-48C5}}

\end{abstract}

\begin{figure*}[b!]
  \vspace{-0.5\baselineskip}
  \centering
  \striplabel[above]{Figures/walker_forward_gallant-pyramid-226.pdf}{teal!20}{Walker --- Ours: F-CIP $+$ forward velocity reward} \\
  \striplabel{Figures/walker_forward_flowing-dragon-252.pdf}{gray!20}{Walker --- Baseline: forward velocity reward only} \\[4pt]
  \striplabel[above]{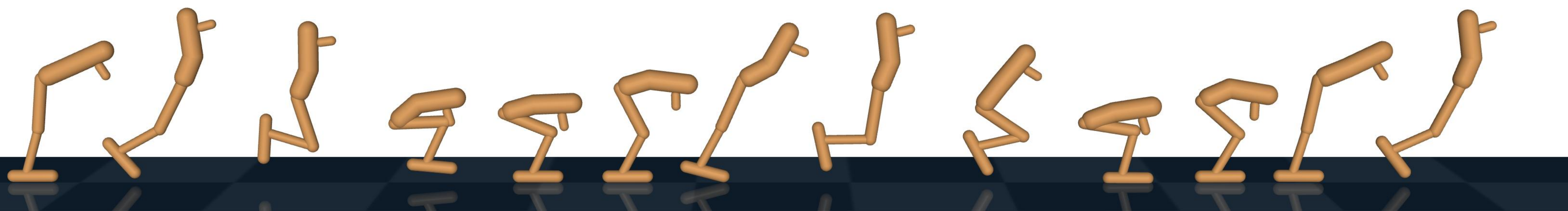}{teal!20}{Hopper --- Ours: F-CIP $+$ forward velocity reward} \\
  \striplabel{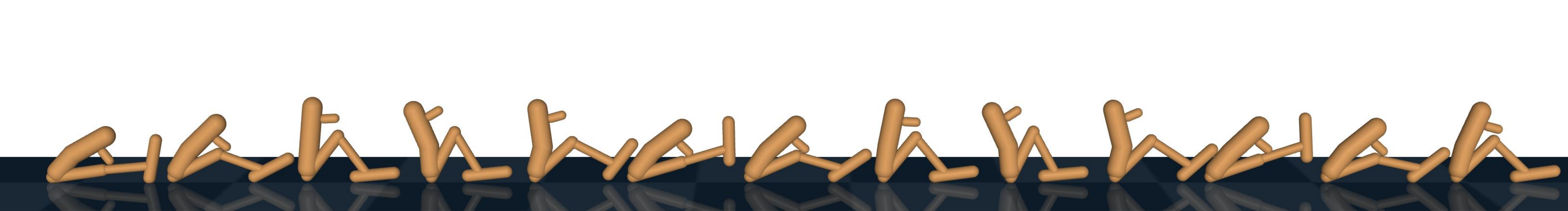}{gray!20}{Hopper --- Baseline: forward velocity reward only}
  \caption{\textbf{Emergent locomotion strategies on walker and hopper.} Training an agent with F-CIP and a forward velocity reward yields upright running and hopping ({\color{teal!70}\textbf{teal}}), whereas the velocity reward alone collapses to bouncing and scooting along the ground ({\color{gray}\textbf{gray}}).}
  \label{fig:locomotion-strips}
\end{figure*}

\section{Introduction}

The field of intrinsic motivation (IM) seeks to identify the goal-agnostic principles of self-motivation in biological systems and to replicate them in artificial ones. One such principle is the \textit{edge of chaos} (EOC) \citep{zhang2024intelligence, tower2024selectively, wissner_2013}: the hypothesis that intelligence arises from an agent's ability to operate at the boundary between stable and chaotic dynamics. In physically embodied agents, this boundary corresponds to states which are maximally unstable yet still controllable through active regulation. The upright stance of a humanoid is the canonical example: it affords walking, running, and hopping, yet it is also the most precarious configuration, since small mistakes readily lead to collapse. The EOC is thus distinct from pure chaos, where perturbations grow faster than a policy can suppress. In this work, we make the EOC directly optimizable as a reinforcement learning (RL) objective.

A plethora of IM objectives have been proposed, yet they are employed almost exclusively in auxiliary roles: as exploration bonuses added to an extrinsic reward \citep{raileanu2020ride, pathak2017curiosity, burda2018exploration}, or as pre-training signals whose value is only realized once a downstream task is imposed \citep{liu2021behaviorvoidunsupervisedactive, eysenbach2018diversityneedlearningskills, laskin2021urlb}. In neither role does the objective require that the behaviors it produces be useful in their own right: curiosity rewards vanish as the learned model improves, and diversity objectives are indifferent to whether the states being explored are dynamically trivial or rich. In both cases the objective is a heuristic designed by human intuition rather than derived from dynamical systems theory, lacking a principle (such as the EOC) that distinguishes significant states from insignificant ones.

Recently \citet{shah2026controllable} introduced Controllable Information Production (CIP), an IM objective that closes this gap. CIP is grounded in the sum of positive open-loop Lyapunov Exponents \citep{wolf1985determining, pesin1977characteristic}, a global property of dynamical systems that governs the entropy rate (in \si{\bit\per\sec}) at which the dynamics generate new distinguishable trajectories. Whereas visitation entropy weights all states alike, CIP weights states by the complexity of the dynamics passing through them, and is therefore maximized not by uniform state coverage but by occupying the most dynamically sensitive regions. A controller maximizing CIP was shown to drive agents to EOC regimes across a suite of benchmark tasks, including raising a humanoid-like system from a passive configuration to an upright stand. This demonstrated that IM grounded in a principled dynamical quantity can serve not merely as an exploration bonus, but rather as the primary objective of the agent.

Existing optimizers of CIP, however, are sampling-based MPC planners \citep{howell2022predictive, pinneri2020sampleefficientcrossentropymethodrealtime, williams2015model} which are ill-suited to the CIP objective. Such planners must forward-simulate a large number of candidate trajectories at every timestep and select the best action sequences. Because evaluating CIP requires running a Riccati recursion over the local Jacobians of each rollout, the cost of every discarded candidate trajectory is far higher than under a standard per-state cost. Gradient-based planners \citep{zhang2025whole, tassa2012synthesis} avoid the need for many rollouts and instead require a differentiable cost evaluated around a local trajectory. Since the CIP objective already depends on the state Jacobian, differentiating it for planning would demand second-order derivatives of the dynamics, which are impractical to obtain in high-dimensional systems. 

RL \citep{sutton1998reinforcement} is a natural alternative to MPC: it amortizes the sampling cost into a learning phase and delivers a reactive feedback policy. The resultant policy is faster to use at test time because it can be run without expensive re-planning at each step. Yet RL imposes two structural requirements on its objective: (i) the reward must be emitted incrementally during the forward agent-environment interaction, and (ii) the reward must be Markovian in the agent's state to support bootstrapped value learning. CIP in its original form violates both. It is computed by a backward-in-time Riccati recursion: the increment at the current time depends on a matrix propagated from the terminal time, i.e., on the entire future trajectory.

We introduce \textbf{Forward CIP} (F-CIP)\footnote{\emph{Throughout, CIP refers to the underlying quantity (the sum of positive open-loop LE) while F-CIP refers to our forward in time per-step decomposition and the reward it exposes; the two coincide in Theorem~\ref{thm:main}.}}: a forward-in-time incremental decomposition of CIP that makes it, for the first time, directly optimizable by RL. This resolves requirement (i) by exposing an instantaneous per-step reward, and requirement (ii) by identifying the information matrix $S_t$ as a sufficient statistic of the past: augmenting the state with $S_t$ exactly restores the Markov property (Proposition~\ref{prop:markov}). We prove (Lemma~\ref{lem:forward_cip}, Theorem~\ref{thm:main}) that the forward iteration converges to the same quantity as the backward recursion, so no approximation is introduced. Empirically, training PPO \citep{schulman2017proximalpolicyoptimizationalgorithms} on the F-CIP reward alone yields unsupervised self-righting and standing. With a single additional velocity term, agents learn upright hopping and running; the same velocity reward without F-CIP collapses to degenerate gaits. 

\textbf{Our contributions are summarized as follows:}
\begin{enumerate}
    \item \textbf{Theory.} A forward-in-time characterization of CIP as a sum of per-step increments.
    \item  \textbf{Algorithm.} A numerically stable log-domain iteration exposing a per-step reward and an information-matrix memory $S_t$, casting CIP maximization as a standard RL problem.
    \item \textbf{Experiments.} The first RL agent trained directly on the CIP objective. Maximizing it drives the unsupervised discovery of primitive behaviors, while combining it with a downstream objective yields coherent locomotion strategies such as hopping and running which otherwise require reward engineering to learn.
\end{enumerate}

\section{Related Work}

Existing IM objectives fall broadly into four families: curiosity, skill discovery, state-entropy maximization, and EOC. We review each and position CIP against them.

\paragraph{Curiosity.}
Curiosity-driven methods reward the surprise of a learned dynamics model, measured as prediction error \citep{pathak2017curiosity, burda2018large}, error of a random target network \citep{burda2018exploration}, or ensemble disagreement \citep{pathak2019selfsupervisedexplorationdisagreement}. These objectives are defined relative to the deficiencies of the learned model rather than a property of the environment itself. Consequently, the reward signal can vanish as the model improves \citep{raileanu2020ride}, and it is maximized by irreducibly unpredictable stimuli regardless of their behavioral relevance, e.g., the noisy-TV problem \citep{burda2018exploration}. CIP instead rewards a dynamical invariant of the system: its signal neither vanishes with learning progress nor depends on model error.

\paragraph{Skill-Discovery.}
Skill-discovery methods \citep{zheng2025can} maximize the mutual information between a latent skill variable and the states that the agent visits under that skill \citep{eysenbach2018diversityneedlearningskills, sharma2019dynamics, gregor2016variationalintrinsiccontrol}. This principle, however, is agnostic to where the state space is behaviorally significant. For example a diverse set of trivially stable skills that partition the state space satisfies the objective as well as a set of dynamically rich ones. CIP takes an alternative approach: it concentrates the agent at the states of maximal dynamical instability rather than spreading it over the entire state space.

\paragraph{State-entropy maximization.}
A related family directly maximizes the entropy of the state-visitation distribution \citep{liu2021behaviorvoidunsupervisedactive, lee2025semdice}. Although both visitation entropy and CIP are entropic quantities, they measure different things: visitation entropy is a static property of the state-visitation distribution, whereas CIP is the rate at which uncertainty grows in open-loop along trajectories induced by the agent. The two are independent: an agent can achieve maximal visitation entropy while residing mostly in regions of low CIP, and conversely can hold a single EOC state of high CIP with near-zero visitation entropy.

\paragraph{Edge of Chaos.}
The closest objectives to our work are also based on the EOC principle. Empowerment \citep{klyubin2005empowerment} is the channel capacity from the agent's actions to its future states; it generalizes the controllability Gramian of linear systems and has been related to the finite-time Lyapunov exponent (LE) spectrum \citep{tiomkin2024intrinsic}. Empowerment, however, has resisted exact factorization into a sum of rewards: channel capacity over a multi-step horizon is a supremum over action distributions that does not separate into per-step increments. Consequently, RL with empowerment has relied on variational lower bounds and fixed-horizon estimators \citep{karl2017unsupervisedrealtimecontrolvariational}. CIP shares the EOC principle but is defined directly at the level of LE, and was previously optimized only by sampling-based MPC for a similar reason: its backward recursion admits no online reward. The present work removes this obstruction.

Similar to our method, \citet{phu_nguyen_2025} also use the sum of positive LEs as a reward, but compute it offline along passive trajectories over a fixed horizon. This design has three limitations. First, precomputing LEs over a large continuous state space is computationally expensive and does not scale. Second, the fixed horizon introduces a truncation hyperparameter. Third, each per-step reward is already a sum of positive LEs over the horizon, so the return accumulates overlapping sums and has no clear interpretation. In F-CIP, the reward is an incremental per-step quantity, and the value function itself learns the sum of positive LEs.

To our knowledge, F-CIP is the first dynamics-grounded intrinsic RL objective with an exact per-step decomposition, requiring neither a variational bound nor a fixed optimization horizon.

\section{Controllable Information Production}

CIP is defined as the sum of positive LE of the open-loop dynamical control system $f: \mathcal{X} \times \mathcal{U} \to \mathcal{X}$ where $\mathcal{U}$ is the action space. Positive open-loop LE quantify the entropy production rate which the agent is exposed to along its trajectory, considering that infinitesimal perturbations grow without regulation due to feedback by a policy $\pi: \mathcal{X} \to \mathcal{U}$.
\begin{equation*}
    \textrm{CIP} := \Lambda^+(f^\mathbf{ol}), \qquad f^{\mathbf{ol}}(x) = f(x, \mathbf{sg}(\pi(x)))
\end{equation*}
Where $\Lambda^+(\cdot)$ is the sum of positive LE of the dynamical system. The $\mathbf{sg}$ symbol denotes the stop-gradient operation. Previously it was shown that CIP can be calculated via a backward recursion:
\begin{equation}\label{eq:backwards_log_domain} 
    \begin{aligned}
        h_t &= h_{t+1} + \log\det(Y_{t+1} + f_{x_t} f_{x_t}^\top), &\qquad h_T &= 0, \\ 
        Y_{t} &= \mathbf{I} - f_{x_t}^\top (Y_{t+1} + f_{x_t} f_{x_t}^\top)^{-1} f_{x_t}, &\qquad Y_T &= \mathbf{I},
    \end{aligned}
\end{equation}
where $f_{x_t} := \tfrac{\partial f}{\partial x} |_{(x_t, u_t)}$ is the local Jacobian with respect to state. If $h_t$ is propagated backward in time along an infinitely long trajectory then its one-half time average converges to $\Lambda^+(f^\mathbf{ol})$:
\begin{equation}
    \lim_{T\to\infty} \frac{1}{2T\Delta t} h_0 = \Lambda^+(f^{\mathbf{ol}}).
\end{equation}
The increment between $h_t$ and $h_{t+1}$ in Equation~\ref{eq:backwards_log_domain} is not available to the agent during the forward rollout as a reward quantity: the source of CIP's incompatibility with RL. Furthermore, it depends on $Y_{t+1}$, which is a function of the entire future trajectory: breaking the Markov Property required for bootstrapped value learning. In the next section we present a dual forward-in-time calculation of CIP which enables training with RL.

\section{Proposed Method for Forward Estimation of CIP}\label{sec:method}

We factor the calculation of CIP into a sum of increments emitted during the forward agent-environment interaction, so that each increment can serve as an instantaneous reward. Our key insight comes from the control-estimation duality \citep{todorov_2008_general}: the backward Riccati recursion in Equation~\ref{eq:backwards_log_domain} is a control-side object, and its estimation-side dual is a forward-in-time covariance iteration whose increments depend only on the current Jacobian $f_{x_t}$ and a propagated information matrix $S_t$. We develop this in four steps: the forward iteration converges to the same open-loop KSE as the backward recursion (Lemma~\ref{lem:forward_cip}); a log-domain form of the iteration is numerically stable and exposes the per-step reward (Lemma~\ref{lem:log_domain}); the augmented state $z_t := (x_t, S_t)$ defines a legitimate MDP (Proposition~\ref{prop:markov}); and the infinite-horizon return of that MDP is exactly CIP (Theorem~\ref{thm:main}).

The uncertainty in the agent's state is given by a covariance matrix $\Sigma_t$. The time evolution of this uncertainty is determined by the local Jacobian as well as the covariance of the process noise $W_t$ injected at each step. The following lemma shows that its log-volume growth rate recovers CIP exactly.

\begin{lemma}[Forward covariance characterization of CIP]\label{lem:forward_cip}
Let:
\begin{equation}\label{eq:cov_prop}
    \Sigma_{t+1} = W_{t+1} + f_{x_t} \Sigma_t f_{x_t}^\top, \qquad \Sigma_{0} = W_0,
\end{equation}
be the open-loop covariance propagation equation for state uncertainty, where $W_t$ is the
process-noise covariance injected at step $t$, uniformly bounded and positive definite. Then:
\begin{equation*}
    \lim_{T\to\infty} \frac{1}{2T\Delta t}\ln\det \Sigma_T = \mathrm{CIP}.
\end{equation*}
\end{lemma}

\begin{proof}
Unrolling Equation~\ref{eq:cov_prop} yeilds a Grammian matrix:
\begin{equation*}
    \Sigma_T = \sum_{t=0}^T \Phi_{t:T}\, W_t\, \Phi_{t:T}^\top,
\end{equation*} with
$\Phi_{t:T} = \prod_{k=t}^{T-1} f_{x_k}$ as a product of Jacobians. Reindexing $\tau = T-t$ and setting $\Psi_\tau := \Phi_{T-\tau:T}^\top$ casts this as the cumulative-sensitivity sum $\Sigma_T = \sum_{\tau=0}^{T} \Psi_\tau^\top\, W_{T-\tau}\, \Psi_\tau$, where each $\Psi_\tau$ is a
length-$\tau$ product of trajectory Jacobians. Because $\Psi_\tau$ and $\Phi_{T-\tau:T}$ share
singular values, the products exhibit the exponential growth $\|\Psi_\tau v_i\| \sim e^{\tau \Delta t \lambda_i^{\mathbf{ol}}}$ governed by open-loop Lyapunov Exponents $\{\lambda_i^\mathbf{ol}\}$. This is exactly the object of \citet[Lemma~A.2]{shah2026controllable}, whose weighting invariance \citep[Lemma~A.3]{shah2026controllable} admits any uniformly bounded positive-definite $\{W_t\}$; together they give $\tfrac{1}{2T\Delta t}\ln\det \Sigma_T \to \sum_{\lambda_i^{\mathbf{ol}}>0}\lambda_i^{\mathbf{ol}} = \Lambda^+(f^\mathbf{ol})$ as $T\to \infty$.
\end{proof}

Direct propagation of Equation~\ref{eq:cov_prop} is numerically unstable, since the
condition number of $\Sigma_t$ grows exponentially with time. Furthermore, there is no additive increment available to interpret as a reward. To overcome this instability we instead propagate $\Sigma_t$ in the $\ln\det$ domain via an iteration derived in Lemma~\ref{lem:log_domain} which exposes an instantaneous reward signal.

\begin{wrapfigure}{r}{0.5\linewidth}
    \vspace{-1em}
    \centering
    \includegraphics[width=\linewidth]{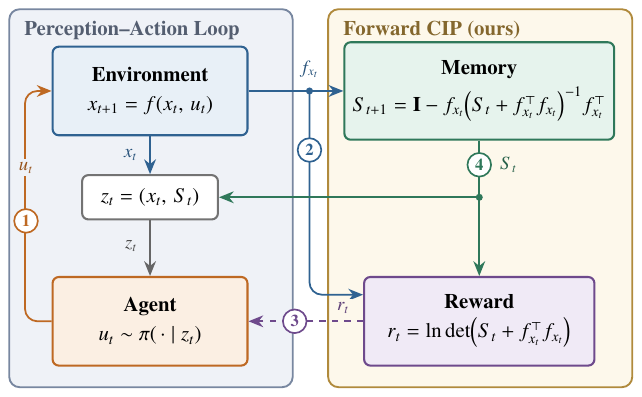}
    \caption{\textbf{Schematic of our method.} The environment returns the Jacobian $f_{x_t}$, which is used to emit reward $r_t$ and then update the memory $S_t$.}
    \label{fig:main-figure}
    \vspace{-1em}
\end{wrapfigure}

\begin{lemma}[Log-domain iteration]\label{lem:log_domain}
Let $h_t := \ln\det \Sigma_t$ and let $S_t := \Sigma_t^{-1}$ be the information matrix, and
choose $W_t = \mathbf{I}$ (admissible by the weighting invariance of Lemma~1). Then:
\begin{align}
    h_{t+1} &= h_t + \ln\det\!\left(S_t + f_{x_t}^\top f_{x_t}\right), & h_0 &= 0, \label{eq:entropy_increment}\\
    S_{t+1} &= \mathbf{I} - f_{x_t}\!\left(S_t + f_{x_t}^\top f_{x_t}\right)^{-1}\! f_{x_t}^\top,
        & S_0 &= \mathbf{I}. \label{eq:memory_increment}
\end{align}
\end{lemma}
\begin{proof}
With $W_t = \mathbf{I}$, Equation~\ref{eq:cov_prop} reduces to
$\Sigma_{t+1} = \mathbf{I} + f_{x_t} \Sigma_t f_{x_t}^\top$ with $\Sigma_0 = \mathbf{I}$.
Taking $\ln\det$ and applying Sylvester's identity yields the first equation; inverting and
applying the Woodbury identity yields the second.
\end{proof}


The iteration in Lemma~\ref{lem:log_domain} is numerically stable: because $\Sigma_t \succeq \mathbf{I}$ for all $t$, the information matrix is bounded, $0 \preceq S_t \preceq \mathbf{I}$, so the exponential growth of $\Sigma_t$ underflows towards zero in $S_t$, remaining well-conditioned as a neural network input.

The increment in Equation~\ref{eq:entropy_increment} defines a per-step reward:
\begin{equation}\label{eq:reward}
    r(x_t, u_t, S_t) := \ln\det\!\left(S_t + f_{x_t}^\top f_{x_t}\right).
\end{equation}
The reward carries the memory term $S_t$ in its signature, so it is not Markovian in $x_t$
alone. The following proposition shows that $S_t$ is a sufficient statistic of the past:
augmenting the state with it recovers a legitimate Markov decision process as stated in the following proposition.

\begin{proposition}[Markov restoration]\label{prop:markov}
Let $z_t := (x_t, S_t) \in \mathcal{Z} := \mathcal{X} \times \mathcal{S}$, where
$\mathcal{S}$ denotes the symmetric matrices satisfying $0 \preceq S \preceq \mathbf{I}$.
Then the tuple $(\mathcal{Z}, \mathcal{U}, p_z, r)$, with $r$ as in
Equation~\ref{eq:reward}, is a Markov decision process.
\end{proposition}

\begin{proof}
The transition of $z_t$ factorizes: $x_{t+1} \sim p(\cdot \mid x_t, u_t)$ is the
environment step, and $S_{t+1}$ in Equation~\ref{eq:memory_increment} is a deterministic
function of $(x_t, u_t, S_t)$ through the state Jacobian $f_{x_t}$. The reward
$r(z_t, u_t)$ depends only on the current augmented state and action. Hence $z_{t+1}$ and
$r_t$ are determined by $(z_t, u_t)$ alone.
\end{proof}

In the following theorem, we state the main result: the return of this MDP is exactly the CIP
estimator of Lemma~\ref{lem:forward_cip}.

\begin{theorem}[Exact decomposition of CIP]\label{thm:main}
In the MDP of Proposition~\ref{prop:markov}, the undiscounted return telescopes,
$\sum_{t=0}^{T-1} r(z_t, u_t) = \ln\det \Sigma_T$, and consequently
\begin{equation}\label{eq:sum_of_rewards}
    \lim_{T\to\infty} \frac{1}{2T\Delta t} \sum_{t=0}^{T-1} r(z_t, u_t)
    = \Lambda^+(f^\mathbf{ol}) = \mathrm{CIP}.
\end{equation}
\end{theorem}

\begin{proof}
By Lemma~\ref{lem:log_domain}, $\sum_{t=0}^{T-1} r_t = h_T - h_0 = \ln\det \Sigma_T$.
The limit follows from Lemma~\ref{lem:forward_cip}. 
\end{proof}

\paragraph{Implementation.}

Theorem~\ref{thm:main} is an average-reward statement, whereas standard policy-gradient
algorithms optimize a discounted objective. In practice we optimize the discounted proxy
with $\gamma$ close to $1$; the discounted-optimal policy converges to the average-reward-optimal policy as $\gamma \to 1$.

The augmented state $z_t = (x_t, S_t)$ is passed to both the value function and the policy.
As $S_t$ is symmetric, only its $n(n{+}1)/2$ unique entries need be carried, and its
boundedness in $(0, \mathbf{I}]$ makes it a well-conditioned network input despite the
augmentation. The Jacobian $f_{x_t}$ is obtained from the Mujoco-MJX \citep{todorov2012mujoco} simulator at each
step. The pseudocode in Algorithm~\ref{alg:cip_adaptation} summarizes the training procedure.

\begin{algorithm}[t]
\caption{Training Forward CIP and Extrinsic Reward}
\label{alg:cip_adaptation}
\begin{algorithmic}[1]
\REQUIRE policy $\pi_\theta$, value function $V_\phi$, iterations $T_{\mathrm{pre}}, T_{\mathrm{adapt}}, T_{\mathrm{warm}}$, extrinsic weight $\beta_{\max}$, horizon $H$, $N$ environments
\FOR{$i := 1, \ldots, T_{\mathrm{pre}} + T_{\mathrm{adapt}}$}
    \STATE $\beta \leftarrow \beta_{\max} \cdot \min\!\left(1, \tfrac{\max(0,\, i - T_{\mathrm{pre}})}{T_{\mathrm{warm}}}\right)$ 
    \FOR{each of $N$ parallel environments}
        \STATE $x_0 \sim \rho_0$, \quad {\color{teal}$S_0 \leftarrow \mathbf{I}$} \hfill {\footnotesize\color{gray} Reset state and memory}
        \FOR{$t := 0, \ldots, H-1$}
            \STATE {\color{teal}$z_t \leftarrow (x_t, S_t)$} \hfill {\footnotesize\color{gray}Proposition\,\ref{prop:markov}}
            \STATE Take action $u_t \sim \pi_\theta(\cdot \mid z_t)$
            \STATE Receive $x_{t+1}$ and $f_{x_t}$
            \STATE {\color{teal}$r_t \leftarrow \ln\det\!\left(S_t + f_{x_t}^\top f_{x_t}\right)$} \hfill {\footnotesize\color{gray}Equation\,\ref{eq:reward}}
            \STATE {\color{teal}$S_{t+1} \leftarrow \mathbf{I} - f_{x_t}(S_t + f_{x_t}^\top f_{x_t})^{-1} f_{x_t}^\top$} \hfill {\footnotesize\color{gray}Equation\,\ref{eq:memory_increment}}
            \STATE $\hat{r}_t \leftarrow r_t + \beta\, v_t$
        \ENDFOR
    \ENDFOR
    \STATE Update $\theta, \phi$ on $\{(z_t, u_t, \hat{r}_t)\}$ \hfill {\footnotesize\color{gray}Standard PPO}
\ENDFOR
\end{algorithmic}
\end{algorithm}

\section{Experiments}

\begin{wrapfigure}{r}{0.5\linewidth}
    \vspace{-6em}
    \centering
    \begin{tikzpicture}
        \node[anchor=south west, inner sep=0] (conv)
            {\includegraphics[width=\linewidth]{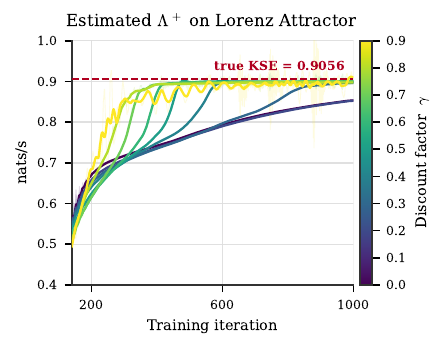}};
        \phantomsubcaption\label{fig:gamma_convergence}

        \begin{scope}[x={(conv.south east)}, y={(conv.north west)}]
            \node[anchor=north west, font=\small] at (0.255,0.86)
                {\bfseries(\subref{fig:gamma_convergence})};

            \node[anchor=south east, inner sep=3pt,
                  fill=white, draw=black!30, line width=0.3pt] (lorenz)
                  at (0.735, 0.27)
                {\includegraphics[width=0.34\linewidth]{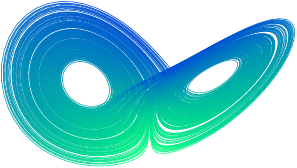}};

            \node[anchor=north west, font=\small, inner sep=1pt]
                at (lorenz.north west) {\bfseries(\subref{fig:lorenz_attractor})};
        \end{scope}
        \phantomsubcaption\label{fig:lorenz_attractor}
    \end{tikzpicture}
    \caption{\textbf{Learning $\Lambda^+$ on the Lorenz attractor}. \textbf{(a)} convergence of the learned value
        function to the known $\Lambda^+$ for a range of
        $\gamma \in [0.0, 0.9]$; \textbf{(b)} the attractor colored by the
        learned $V_\phi(z)$ for $\gamma = 0.9$ where bright green indicates higher value.}
    \label{fig:lorenz}
    \vspace{-2em}
\end{wrapfigure}

Our experiments address three questions: \textbf{Q1.}~Does a learned value function converge to the true $\Lambda^+$ on a known attractor? 
\textbf{Q2.}~What behaviors are induced by maximizing F-CIP, with no extrinsic reward? \textbf{Q3.}~Can F-CIP simplify reward engineering on locomotion tasks?

\paragraph{\textbf{Q1.} $\Lambda^+$ Estimation}

We validate the convergence of our learned incremental estimator on the Lorenz attractor \citep{lorenz2017deterministic} with a known $\Lambda^+$ value of $0.906$ nats/s \citep{sprott_lorenz}. The $\Lambda^+$ is extracted from the learned estimator as $\tfrac{1-\gamma}{2\Delta t} \mathbb{E}[V_\phi(z_t)]$, where the prefactor converts the accumulated discounted return into a rate. Here $V_\phi$ is estimated by TD learning \citep{sutton1998reinforcement} on trajectories sampled from the attractor. Figure~\ref{fig:gamma_convergence} demonstrates that the estimator converges to the known value for a range of discount factors $\gamma$. Setting $\gamma$ to zero still enables estimation of $\Lambda^+$ since that special case corresponds to the value function learning only the instantaneous reward. However, we observe that increasing values of $\gamma$ result in faster convergence to the true $\Lambda^+$. We visualize the output of the value function over a trajectory of the attractor for $\gamma = 0.9$ in Figure~\ref{fig:lorenz_attractor} where bright green coloration corresponds to larger values. Since the value function learns a discounted summation its output can be interpreted as a local indicator of the ``chaos-to-go''. The value function attains its largest values near the splitting point between the two attractor wings, where the local dynamics are most unstable and trajectories can diverge towards either wing. This property motivates using the learned value function as a local signal of proximity to the EOC.

\paragraph{\textbf{Q2.} Emergent Behaviors}

\begin{table}
  \centering
  \setlength{\tabcolsep}{4pt}
  \begin{tabular}{@{}lcccccc@{}}
    \toprule
    &
    \multicolumn{6}{c}{Intrinsic Motivation Objective} \\
    \cmidrule(lr){2-7}
    System
      & \makecell{F-CIP\\(PPO)}
      & \makecell{F-CIP\\(APPO)}
      & DIAYN
      & SMM
      & ICM
      & APT \\
    \midrule
    Cart Pole & $\mathbf{0.99}$ & $\mathbf{0.99}$ & $0.26$ & $0.68$ & $0.07$ & $0.64$ \\
    Double Pendulum & $\mathbf{0.99}$ & $\mathbf{0.99}$ & $0.44$ & $0.06$ & $0.05$ & $0.25$ \\
    Triple Pendulum & $\mathbf{0.99}$ & $\mathbf{0.99}$ & $0.38$ & $0.14$ & $0.09$  & $0.48$ \\
    Gibbon  & $\mathbf{0.98}$ & $\mathbf{0.99}$ & $0.51$ & $0.52$ & $0.05$  & $0.41$ \\
    \bottomrule
  \end{tabular}
  \caption{\textbf{Comparison of IM on pendulum environments}. We report mean extremity height over the final $10$\% of the episode averaged over $10$ random seeds. Values are min-max normalized so that $0.0$ corresponds to the initial hanging configuration and $1.0$ corresponds to the fully upright pose. 
  } \label{tab:results}
\end{table}

Maximizing F-CIP alone, with no extrinsic reward, produces swing-up and upright stabilization: Figure~\ref{fig:trajectories} shows a full swing-up and stabilization trajectory in the gibbon environment. Table~\ref{tab:results} quantifies this across all four systems, reporting the normalized height of each system's extremity (i.e. pendulum tip; gibbon head) over the final $10\%$ of the episode, averaged over $10$ seeds. All agents start from the hanging pose, so the intrinsic objective alone must drive both swing-up and stabilization. F-CIP is the only objective that reliably reaches the upright configuration on all four systems. The skill-based baseline methods (DIAYN and SMM) are evaluated on their best skill, i.e., the one attaining the highest average height. Finally, since the F-CIP objective in Equation~\ref{eq:sum_of_rewards} is naturally an average-reward quantity, we also train an average-reward PPO variant (APPO)~\citep{ma2021averagerewardreinforcementlearningtrust}, which matches the reliability of standard PPO.

\begin{figure*}[t]
    \centering
    \striplabel{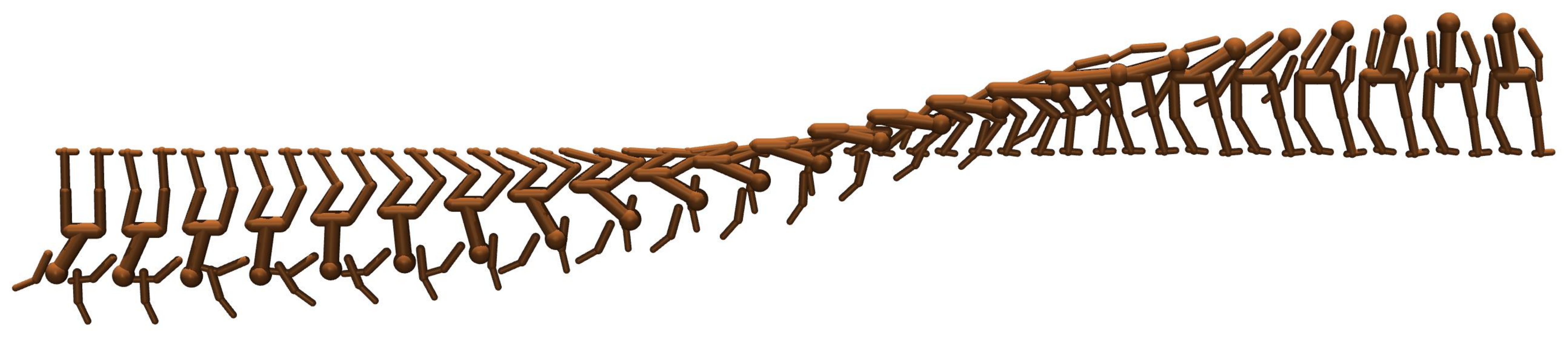}{teal!20}{Gibbon --- Ours: F-CIP pre-training with no extrinsic reward}
    \caption{\textbf{Unsupervised standing up behavior via F-CIP maximization}. This behavior is achieved through maximization of F-CIP alone without any extrinsic reward.}
    \label{fig:trajectories}
\end{figure*}

\paragraph{\textbf{Q3.} Reward Simplification}

\begin{table*}
\centering
\scriptsize
\begin{tabular}{lcccccc}
\toprule
& \multicolumn{3}{c}{\textbf{Hopper}} & \multicolumn{3}{c}{\textbf{Walker}} \\
\cmidrule(lr){2-4}\cmidrule(lr){5-7}
Method & \makecell{Task Return $\uparrow$} & \makecell{Fwd.\ Vel.\ (m/s)} & \makecell{Height (m)}
       & \makecell{Task Return $\uparrow$} & \makecell{Fwd.\ Vel.\ (m/s)} & \makecell{Height (m)} \\
\midrule
Ours                          & $\mathbf{0.39 \pm 0.38}^*$ & $1.47 \pm 0.33$ & $0.73 \pm 0.15$ & $\mathbf{0.60 \pm 0.03}$ & $4.53 \pm 0.11$ & $1.19 \pm 0.04$ \\
Fwd.\ Vel.\ only              & $0.00 \pm 0.00$ & $0.77 \pm 0.16$ & $0.33 \pm 0.03$ & $0.10 \pm 0.04$ & $4.41 \pm 0.22$ & $0.46 \pm 0.10$ \\
ICM$^\dagger$                 & $0.00 \pm 0.00$ & $1.52 \pm 0.55$ & $0.22 \pm 0.05$ & $0.09 \pm 0.04$ & $3.85 \pm 0.41$ & $0.48 \pm 0.12$ \\
APT$^\dagger$                 & $0.00 \pm 0.00$ & $1.63 \pm 0.29$ & $0.23 \pm 0.07$ & $0.10 \pm 0.05$ & $3.38 \pm 1.14$ & $0.47 \pm 0.14$ \\
DIAYN$^\dagger$               & $0.00 \pm 0.00$ & $1.22 \pm 0.17$ & $0.29 \pm 0.12$ & $0.08 \pm 0.03$ & $3.29 \pm 0.33$ & $0.42 \pm 0.09$ \\
SMM$^\dagger$                 & $0.00 \pm 0.00$ & $1.12 \pm 0.10$ & $0.23 \pm 0.04$ & $0.09 \pm 0.04$ & $4.53 \pm 0.51$ & $0.45 \pm 0.10$ \\
\midrule
Full reward \textit{(oracle)} & $0.52 \pm 0.12$ & $1.04 \pm 0.23$   & $0.79 \pm 0.09$ & $0.78 \pm 0.04$ & $6.07 \pm 0.33$ & $1.13 \pm 0.02$ \\
\bottomrule
\end{tabular}
\caption{\textbf{Reward simplification on Hopper and Walker.} Mean and standard deviation calculated over $10$ evaluations of the trained policies. Ours pre-trains with F-CIP alone, then fine-tunes on the forward-velocity reward while continuing to maximize F-CIP (Algorithm~\ref{alg:cip_adaptation}). Task return is the comparison metric and is measured with the full hand-designed reward; \textbf{bold} marks the best performance among the methods. Forward velocity and torso height are reported as diagnostics of the emergent gait. 
Each pretraining baseline ($\dagger$) is run in its intended configuration: the intrinsic objective is maximized during pretraining and then replaced by the velocity reward during fine-tuning. In contrast, Ours continues to maximize F-CIP throughout fine-tuning. \textit{($*$) The high variance on Hopper arises because the oracle reward is zero whenever the torso drops below a fixed height threshold. All 10 seeds hop upright, but some adopt a deeper crouch that falls below this threshold. This strategy allows our policy to hop faster than the oracle ($1.47$ vs. $1.04$ \text{m/s}).}
}
\label{tab:reward-simplification}
\end{table*}

Inducing coherent locomotion strategies such as hopping and walking using RL usually requires careful reward engineering. The simplest reward of maximizing forward velocity describes the desired outcome to move forward but imposes no bias on how that should be achieved. Maximizing a forward velocity reward with an RL agent results in degenerate solutions such as ``scooting'' along the ground (gray rows of Figure~\ref{fig:locomotion-strips}) rather than adopting an upright gait. The standard remedy is to design a reward to incentivize upright posture, requiring domain knowledge about the robot itself. We show that CIP can supply the inductive bias to maintain an upright posture, eliminating the need for this domain knowledge.

Our procedure has two phases (Algorithm~\ref{alg:cip_adaptation}). We first pre-train the agent for $T_\textrm{pre}$ steps to maintain a stable and upright posture using the F-CIP reward (Equation~\ref{eq:reward}). Next, we adapt the pre-trained policy for $T_\textrm{adapt}$ steps by exposing it to a forward velocity reward weighted by a $\beta$ parameter, which warms up from zero towards $\beta_{\max}$ over $T_\textrm{warm}$ steps. The agent optimizes the combined objective: $\hat{r}_t = r_t + \beta v_t$, where $v_t$ is the forward velocity component. Crucially, \textit{our method never optimizes the oracle hand-designed reward}; that reward is only used at test time to evaluate task return.


Table~\ref{tab:reward-simplification} reports rollout statistics for both tasks. All baselines collapse to scooting: they hold a low torso height and accrue near-zero task return. Our method instead learns to hop and to run (see the {\color{teal!70}\textbf{teal}} rows of Figure~\ref{fig:locomotion-strips}), maintaining an upright posture and recovering the majority of the oracle's task return despite never observing that reward.

On hopper, the large variance of our task return is an artifact of the oracle metric rather than of the gait: the oracle reward is gated to zero whenever the torso leaves a narrow height band, and while all $10$ evaluation seeds hop upright, some hop at heights outside this band and score zero despite a correct gait. Notably, our hopping policy travels ${\sim}40\%$ faster than the oracle policy ($1.47$ vs.\ $1.04$ m/s), suggesting that the hand-designed height band could be overly restrictive: it penalizes the torso height excursions that fast hopping requires. On walker, our gait falls within the oracle's height band, yielding correspondingly low variance. Its forward velocity is comparable to the fastest baselines, yet ours is the only method besides the oracle to achieve an upright running gait whereas the baselines attain their speed by scooting.

\section{Discussion}\label{sec:discussion}


We introduced F-CIP, a practical scheme for maximizing CIP with standard RL algorithms. Whereas CIP was previously optimizable only by expensive MPC planners, our exact per-step decomposition makes it compatible with bootstrapped value learning: the learned value function recovers the true $\Lambda^+$ of a known chaotic attractor (\textbf{Q1.}), maximizing the F-CIP reward alone produces self-righting and upright stabilization where standard IM objectives, do not (\textbf{Q2.}), and F-CIP pre-training replaces the postural reward engineering that locomotion tasks otherwise require (\textbf{Q3.}).

These results support a broader position: intrinsic motivation need not be confined to an exploration bonus. Traditionally in continuous control, IM objectives serve almost exclusively as auxiliary signals. This aligns with their underlying design, as principles like coverage and surprise are inherently geared toward exploration. CIP is a different IM signal: because it targets EOC states directly, it produces useful behavior in both roles. As a primary objective it discovers self-righting with no external signal; as an auxiliary term it supplies an inductive bias toward controllable instability that reward engineers otherwise encode by hand.

Our method has three main limitations. First, the reward requires analytical Jacobians. Differentiable simulators such as MuJoCo-MJX \citep{todorov2012mujoco} provide these efficiently, and F-CIP is well suited to this setting; physical reality, however, offers no Jacobians, so sim-to-real transfer will require either training entirely in simulation or learning the Jacobians from data. Second, the memory matrix $S_t$ grows quadratically in the state dimension. On-policy training requires every parallel environment to carry its own $S_t$, and off-policy training would require storing it per-transition in a replay buffer, so scaling to high-dimensional robots demands a compressed representation. Third, CIP admits undesireable maximizers, analogous to the noisy-TV problem of curiosity \citep{burda2018exploration, burda2018large}: in the pendulum systems the agent can accumulate reward by spinning freely, and in the hopper by uncontrolled thrashing. Adding joint damping removed these optima in our environments: damping taxes indiscriminate velocity accumulation while leaving configuration-dependent instability rewarded. However, this is an environment-level modification rather than an objective-level one, and a principled resolution remains open.

Each limitation suggests future work. Jacobians can in principle be extracted from a learned world model: one-step forward models of the kind trained by curiosity methods \citep{pathak2017curiosity} are differentiable by construction, which would make F-CIP fully learning-based and simulator-agnostic. The $S_t$ bottleneck can be addressed with feature extractors that respect the symmetric positive-semi-definite structure of the memory, such as SPDNet \citep{huang2017riemannian}, mirroring the use of pretrained encoders for image observations in RL \citep{parisi2022unsurprisingeffectivenesspretrainedvision}. Finally, we plan to investigate alternative formulations of CIP to isolate its desirable maximizers via principled restrictions on the agent's control power and velocities.



\section*{AI Disclosure}

In this work, we used generative AI tools for implementation: they assisted in writing some functions in the codebase and in building the project webpage. We have not used generative AI tools for data generation, proof formulation, methodology design, or result interpretation. Additionally, we used generative AI tools for formatting figures and tables in \LaTeX{} and for polishing the writing for grammar and clarity. We reviewed all AI-assisted work: code and webpage were checked and tested by the authors, and all edited text was verified to preserve its original meaning. We take full responsibility for the final content.
\section*{Ethics Statement}

This work studies intrinsic motivation for reinforcement learning in
simulated physical systems (pendulum benchmarks and the DeepMind Control
Suite hopper and walker). It involves no human subjects, personal data, or
sensitive information, and all experiments are conducted in simulation.
We do not foresee direct negative societal impacts. As with any advance in
autonomous control, the methods could eventually contribute to robotic
systems whose deployment raises safety considerations; we note that F-CIP
encourages controllable rather than uncontrolled behavior, and that
sim-to-real transfer remains an open problem for our method
(Section~\ref{sec:discussion}). The authors have read and adhere to the
ICLR Code of Ethics.

\section*{Reproducibility Statement}

All experiments are conducted in simulation using publicly available
environments: the pendulum and gibbon benchmarks of \citet{shah2026controllable}
and the hopper and walker tasks from the DeepMind Control Suite
\citep{Tunyasuvunakool_2020}, with the modifications (joint
damping, gear strength, reset distributions) described in
Appendix~\ref{app:environments}. Section~\ref{sec:method} contains complete
proofs of Lemmas~\ref{lem:forward_cip} and~\ref{lem:log_domain},
Proposition~\ref{prop:markov}, and Theorem~\ref{thm:main}, which together
define the F-CIP reward (Equation~\ref{eq:reward}) and memory update
(Equation~\ref{eq:memory_increment}). Algorithm~\ref{alg:cip_adaptation}
gives the full training procedure, with per-environment hyperparameters in
Table~\ref{tab:hyperparams}, the training and evaluation protocol in
Appendix~\ref{app:training}, and compute requirements in
Appendix~\ref{app:compute}. All reported results are averaged over 10
random seeds. Source code, environment configurations, and training
scripts are provided in an anonymized repository at \href{https://anonymous.4open.science/r/cip_rl-23B3/README.md}{\color{magenta}\nolinkurl{anonymous.4open.science/r/cip_rl-23B3/README.md}}.






\bibliography{iclr2027_conference}

\begin{thebibliography}{38}
\providecommand{\natexlab}[1]{#1}
\providecommand{\url}[1]{\texttt{#1}}
\expandafter\ifx\csname urlstyle\endcsname\relax
  \providecommand{\doi}[1]{doi: #1}\else
  \providecommand{\doi}{doi: \begingroup \urlstyle{rm}\Url}\fi

\bibitem[Burda et~al.(2019{\natexlab{a}})Burda, Edwards, Pathak, Storkey, Darrell, and Efros]{burda2018large}
Yuri Burda, Harri Edwards, Deepak Pathak, Amos Storkey, Trevor Darrell, and Alexei~A. Efros.
\newblock Large-scale study of curiosity-driven learning.
\newblock In \emph{International Conference on Learning Representations}, 2019{\natexlab{a}}.

\bibitem[Burda et~al.(2019{\natexlab{b}})Burda, Edwards, Storkey, and Klimov]{burda2018exploration}
Yuri Burda, Harrison Edwards, Amos Storkey, and Oleg Klimov.
\newblock Exploration by random network distillation.
\newblock In \emph{International Conference on Learning Representations}, 2019{\natexlab{b}}.
\newblock URL \url{https://openreview.net/forum?id=H1lJJnR5Ym}.

\bibitem[Eysenbach et~al.(2019)Eysenbach, Gupta, Ibarz, and Levine]{eysenbach2018diversityneedlearningskills}
Benjamin Eysenbach, Abhishek Gupta, Julian Ibarz, and Sergey Levine.
\newblock Diversity is all you need: Learning skills without a reward function.
\newblock In \emph{International Conference on Learning Representations}, 2019.
\newblock URL \url{https://iclr.cc/virtual/2019/poster/720}.
\newblock Poster.

\bibitem[Gregor et~al.(2016)Gregor, Rezende, and Wierstra]{gregor2016variationalintrinsiccontrol}
Karol Gregor, Danilo~Jimenez Rezende, and Daan Wierstra.
\newblock Variational intrinsic control, 2016.
\newblock URL \url{https://arxiv.org/abs/1611.07507}.

\bibitem[Howell et~al.(2022)Howell, Gileadi, Tunyasuvunakool, Zakka, Erez, and Tassa]{howell2022predictive}
Taylor Howell, Nimrod Gileadi, Saran Tunyasuvunakool, Kevin Zakka, Tom Erez, and Yuval Tassa.
\newblock Predictive sampling: Real-time behaviour synthesis with mujoco.
\newblock \emph{arXiv preprint arXiv:2212.00541}, 2022.

\bibitem[Huang \& Van~Gool(2017)Huang and Van~Gool]{huang2017riemannian}
Zhiwu Huang and Luc Van~Gool.
\newblock A riemannian network for spd matrix learning.
\newblock In \emph{Proceedings of the AAAI conference on artificial intelligence}, volume~31, 2017.

\bibitem[Karl et~al.(2022)Karl, Becker-Ehmck, Soelch, Benbouzid, van~der Smagt, and Bayer]{karl2017unsupervisedrealtimecontrolvariational}
Maximilian Karl, Philip Becker-Ehmck, Maximilian Soelch, Djalel Benbouzid, Patrick van~der Smagt, and Justin Bayer.
\newblock Unsupervised real-time control through variational empowerment.
\newblock In Tamim Asfour, Eiichi Yoshida, Jaeheung Park, Henrik Christensen, and Oussama Khatib (eds.), \emph{Robotics Research}, pp.\  158--173, Cham, 2022. Springer International Publishing.
\newblock ISBN 978-3-030-95459-8.

\bibitem[Klyubin et~al.(2005)Klyubin, Polani, and Nehaniv]{klyubin2005empowerment}
Alexander~S Klyubin, Daniel Polani, and Chrystopher~L Nehaniv.
\newblock Empowerment: A universal agent-centric measure of control.
\newblock In \emph{2005 ieee congress on evolutionary computation}, volume~1, pp.\  128--135. IEEE, 2005.

\bibitem[Laskin et~al.(2021)Laskin, Yarats, Liu, Lee, Zhan, Lu, Cang, Pinto, and Abbeel]{laskin2021urlb}
Michael Laskin, Denis Yarats, Hao Liu, Kimin Lee, Albert Zhan, Kevin Lu, Catherine Cang, Lerrel Pinto, and Pieter Abbeel.
\newblock Urlb: Unsupervised reinforcement learning benchmark.
\newblock \emph{arXiv preprint arXiv:2110.15191}, 2021.

\bibitem[Lee et~al.(2025)Lee, Sun, and Abbeel]{lee2025semdice}
Jongmin Lee, Meiqi Sun, and Pieter Abbeel.
\newblock Semdice: Off-policy state entropy maximization via stationary distribution correction estimation.
\newblock \emph{arXiv preprint arXiv:2512.10042}, 2025.

\bibitem[Lee et~al.(2020)Lee, Eysenbach, Parisotto, Xing, Levine, and Salakhutdinov]{lee2020efficientexplorationstatemarginal}
Lisa Lee, Benjamin Eysenbach, Emilio Parisotto, Eric Xing, Sergey Levine, and Ruslan Salakhutdinov.
\newblock Efficient exploration via state marginal matching, 2020.
\newblock URL \url{https://arxiv.org/abs/1906.05274}.

\bibitem[Liu \& Abbeel(2021)Liu and Abbeel]{liu2021behaviorvoidunsupervisedactive}
Hao Liu and Pieter Abbeel.
\newblock Behavior from the void: Unsupervised active pre-training, 2021.
\newblock URL \url{https://arxiv.org/abs/2103.04551}.

\bibitem[Lorenz(2017)]{lorenz2017deterministic}
Edward~N Lorenz.
\newblock Deterministic nonperiodic flow 1.
\newblock In \emph{Universality in Chaos, 2nd edition}, pp.\  367--378. Routledge, 2017.

\bibitem[Ma et~al.(2021)Ma, Tang, Xia, Yang, and Zhao]{ma2021averagerewardreinforcementlearningtrust}
Xiaoteng Ma, Xiaohang Tang, Li~Xia, Jun Yang, and Qianchuan Zhao.
\newblock Average-reward reinforcement learning with trust region methods, 2021.
\newblock URL \url{https://arxiv.org/abs/2106.03442}.

\bibitem[Nguyen et~al.(2025)Nguyen, Polani, and Tiomkin]{phu_nguyen_2025}
Phu Nguyen, Daniel Polani, and Stas Tiomkin.
\newblock Suple: Robot learning with lyapunov rewards.
\newblock In \emph{2025 IEEE International Conference on Robotics and Automation (ICRA)}, pp.\  1177--1183, 2025.
\newblock \doi{10.1109/ICRA55743.2025.11128350}.

\bibitem[Parisi et~al.(2022)Parisi, Rajeswaran, Purushwalkam, and Gupta]{parisi2022unsurprisingeffectivenesspretrainedvision}
Simone Parisi, Aravind Rajeswaran, Senthil Purushwalkam, and Abhinav Gupta.
\newblock The unsurprising effectiveness of pre-trained vision models for control, 2022.
\newblock URL \url{https://arxiv.org/abs/2203.03580}.

\bibitem[Pathak et~al.(2017)Pathak, Agrawal, Efros, and Darrell]{pathak2017curiosity}
Deepak Pathak, Pulkit Agrawal, Alexei~A Efros, and Trevor Darrell.
\newblock Curiosity-driven exploration by self-supervised prediction.
\newblock In \emph{International conference on machine learning}, pp.\  2778--2787. PMLR, 2017.

\bibitem[Pathak et~al.(2019)Pathak, Gandhi, and Gupta]{pathak2019selfsupervisedexplorationdisagreement}
Deepak Pathak, Dhiraj Gandhi, and Abhinav Gupta.
\newblock Self-supervised exploration via disagreement, 2019.
\newblock URL \url{https://arxiv.org/abs/1906.04161}.

\bibitem[Pesin(1977)]{pesin1977characteristic}
Ya~B Pesin.
\newblock Characteristic lyapunov exponents and smooth ergodic theory.
\newblock \emph{Russian Mathematical Surveys}, 32\penalty0 (4):\penalty0 55, 1977.

\bibitem[Pinneri et~al.(2021)Pinneri, Sawant, Blaes, Achterhold, Stueckler, Rolinek, and Martius]{pinneri2020sampleefficientcrossentropymethodrealtime}
Cristina Pinneri, Shambhuraj Sawant, Sebastian Blaes, Jan Achterhold, Joerg Stueckler, Michal Rolinek, and Georg Martius.
\newblock Sample-efficient cross-entropy method for real-time planning.
\newblock In \emph{Conference on Robot Learning}, pp.\  1049--1065. PMLR, 2021.

\bibitem[Raileanu \& Rockt{\"a}schel(2020)Raileanu and Rockt{\"a}schel]{raileanu2020ride}
Roberta Raileanu and Tim Rockt{\"a}schel.
\newblock Ride: Rewarding impact-driven exploration for procedurally-generated environments.
\newblock \emph{arXiv preprint arXiv:2002.12292}, 2020.

\bibitem[Schulman et~al.(2017)Schulman, Wolski, Dhariwal, Radford, and Klimov]{schulman2017proximalpolicyoptimizationalgorithms}
John Schulman, Filip Wolski, Prafulla Dhariwal, Alec Radford, and Oleg Klimov.
\newblock Proximal policy optimization algorithms, 2017.
\newblock URL \url{https://arxiv.org/abs/1707.06347}.

\bibitem[Shah \& Tiomkin(2026)Shah and Tiomkin]{shah2026controllable}
Tristan Shah and Stas Tiomkin.
\newblock Emergence of physical intelligence via controllable information production.
\newblock In \emph{Advances in Neural Information Processing Systems}, volume 39, Main Conference, 2026.
\newblock Oral presentation.

\bibitem[Sharma et~al.(2019)Sharma, Gu, Levine, Kumar, and Hausman]{sharma2019dynamics}
Archit Sharma, Shixiang Gu, Sergey Levine, Vikash Kumar, and Karol Hausman.
\newblock Dynamics-aware unsupervised discovery of skills.
\newblock \emph{arXiv preprint arXiv:1907.01657}, 2019.

\bibitem[Sprott(1997)]{sprott_lorenz}
Julien~Clinton Sprott.
\newblock Lyapunov exponent and dimension of the lorenz attractor.
\newblock \url{https://sprott.physics.wisc.edu/chaos/lorenzle.htm}, 1997.
\newblock Accessed: 2026-05-01.

\bibitem[Sutton et~al.(1998)Sutton, Barto, et~al.]{sutton1998reinforcement}
Richard~S Sutton, Andrew~G Barto, et~al.
\newblock \emph{Reinforcement learning: An introduction}, volume~1.
\newblock MIT press Cambridge, 1998.

\bibitem[Tassa et~al.(2012)Tassa, Erez, and Todorov]{tassa2012synthesis}
Yuval Tassa, Tom Erez, and Emanuel Todorov.
\newblock Synthesis and stabilization of complex behaviors through online trajectory optimization.
\newblock In \emph{2012 IEEE/RSJ International Conference on Intelligent Robots and Systems}, pp.\  4906--4913. IEEE, 2012.

\bibitem[Tiomkin et~al.(2024)Tiomkin, Nemenman, Polani, and Tishby]{tiomkin2024intrinsic}
Stas Tiomkin, Ilya Nemenman, Daniel Polani, and Naftali Tishby.
\newblock Intrinsic motivation in dynamical control systems.
\newblock \emph{PRX Life}, 2\penalty0 (3):\penalty0 033009, 2024.

\bibitem[Todorov(2008)]{todorov_2008_general}
Emanuel Todorov.
\newblock General duality between optimal control and estimation.
\newblock In \emph{2008 47th IEEE Conference on Decision and Control}, pp.\  4286--4292, 2008.
\newblock \doi{10.1109/CDC.2008.4739438}.

\bibitem[Todorov et~al.(2012)Todorov, Erez, and Tassa]{todorov2012mujoco}
Emanuel Todorov, Tom Erez, and Yuval Tassa.
\newblock Mujoco: A physics engine for model-based control.
\newblock In \emph{2012 IEEE/RSJ International Conference on Intelligent Robots and Systems}, pp.\  5026--5033. IEEE, 2012.
\newblock \doi{10.1109/IROS.2012.6386109}.

\bibitem[Tower(2024)]{tower2024selectively}
John Tower.
\newblock Selectively advantageous instability in biotic and pre-biotic systems and implications for evolution and aging.
\newblock \emph{Frontiers in Aging}, 5:\penalty0 1376060, 2024.

\bibitem[Tunyasuvunakool et~al.(2020)Tunyasuvunakool, Muldal, Doron, Liu, Bohez, Merel, Erez, Lillicrap, Heess, and Tassa]{Tunyasuvunakool_2020}
Saran Tunyasuvunakool, Alistair Muldal, Yotam Doron, Siqi Liu, Steven Bohez, Josh Merel, Tom Erez, Timothy Lillicrap, Nicolas Heess, and Yuval Tassa.
\newblock dm\_control: Software and tasks for continuous control.
\newblock \emph{Software Impacts}, 6:\penalty0 100022, November 2020.
\newblock ISSN 2665-9638.
\newblock \doi{10.1016/j.simpa.2020.100022}.
\newblock URL \url{http://dx.doi.org/10.1016/j.simpa.2020.100022}.

\bibitem[Williams et~al.(2015)Williams, Aldrich, and Theodorou]{williams2015model}
Grady Williams, Andrew Aldrich, and Evangelos Theodorou.
\newblock Model predictive path integral control using covariance variable importance sampling.
\newblock \emph{arXiv preprint arXiv:1509.01149}, 2015.

\bibitem[Wissner-Gross \& Freer(2013)Wissner-Gross and Freer]{wissner_2013}
A.~D. Wissner-Gross and C.~E. Freer.
\newblock Causal entropic forces.
\newblock \emph{Phys. Rev. Lett.}, 110:\penalty0 168702, Apr 2013.
\newblock \doi{10.1103/PhysRevLett.110.168702}.
\newblock URL \url{https://link.aps.org/doi/10.1103/PhysRevLett.110.168702}.

\bibitem[Wolf et~al.(1985)Wolf, Swift, Swinney, and Vastano]{wolf1985determining}
Alan Wolf, Jack~B Swift, Harry~L Swinney, and John~A Vastano.
\newblock Determining lyapunov exponents from a time series.
\newblock \emph{Physica D: nonlinear phenomena}, 16\penalty0 (3):\penalty0 285--317, 1985.

\bibitem[Zhang et~al.(2025{\natexlab{a}})Zhang, Howell, Yi, Pan, Shi, Qu, Erez, Tassa, and Manchester]{zhang2025whole}
John~Z Zhang, Taylor~A Howell, Zeji Yi, Chaoyi Pan, Guanya Shi, Guannan Qu, Tom Erez, Yuval Tassa, and Zachary Manchester.
\newblock Whole-body model-predictive control of legged robots with mujoco.
\newblock \emph{arXiv preprint arXiv:2503.04613}, 2025{\natexlab{a}}.

\bibitem[Zhang et~al.(2025{\natexlab{b}})Zhang, Patel, Rizvi, Liu, He, Karbasi, Zappala, and van Dijk]{zhang2024intelligence}
Shiyang Zhang, Aakash Patel, Syed Rizvi, Nianchen Liu, Sizhuang He, Amin Karbasi, Emanuele Zappala, and David van Dijk.
\newblock Intelligence at the edge of chaos.
\newblock In Y.~Yue, A.~Garg, N.~Peng, F.~Sha, and R.~Yu (eds.), \emph{International Conference on Learning Representations}, volume 2025, pp.\  86576--86592, 2025{\natexlab{b}}.
\newblock URL \url{https://proceedings.iclr.cc/paper_files/paper/2025/file/d791394d32c428aecc7a5b101fb47799-Paper-Conference.pdf}.

\bibitem[Zheng et~al.(2025)Zheng, Tuyls, Peng, and Eysenbach]{zheng2025can}
Chongyi Zheng, Jens Tuyls, Joanne Peng, and Benjamin Eysenbach.
\newblock Can a misl fly? analysis and ingredients for mutual information skill learning.
\newblock In \emph{International Conference on Learning Representations}, volume 2025, pp.\  31876--31903, 2025.

\end{thebibliography}
\bibliographystyle{iclr2027_conference}

\appendix
\section{Appendix}

\subsection{Algorithm}\label{app:algorithm}

Algorithm~\ref{alg:cip_adaptation} outlines the full training procedure. The initial-state distribution $\rho_0$ is the reset distribution: the hanging configuration for the pendulum systems and the default reset for the locomotion environments. The pure-intrinsic experiment (Q2) runs only the first $T_{\mathrm{pre}}$ iterations, during which $\beta = 0$; the locomotion experiment (Q3) additionally runs $T_{\mathrm{adapt}}$ adaptation iterations with $\beta$ annealed from $0$ to $\beta_{\max}$ over $T_{\mathrm{warm}}$ iterations.

\subsection{Hyperparameters}

\begin{table}
\centering
\caption{Hyperparameters. Environments: CP (cart pole), DP (double
pendulum), TP (triple pendulum), G (gibbon), H (hopper), W (walker).
$T_{\mathrm{pre}}$, $T_{\mathrm{adapt}}$, and $T_{\mathrm{warm}}$ are
measured in PPO iterations, and $\beta_{\max}$ is the maximum
extrinsic reward coefficient. Batch size denotes the
minibatch size used for each gradient step.}
\label{tab:hyperparams}
\small
\setlength{\tabcolsep}{4pt}
\begin{tabular}{@{}lcccccc@{}}
\toprule
 & CP & DP & TP & G & H & W \\
\midrule
\multicolumn{7}{@{}l}{\textit{Environment-specific}}\\
$T_{\mathrm{pre}}$   & 3500  & 3800  & 10000 & 10000 & 6000  & 1500  \\
$T_{\mathrm{adapt}}$ & 0     & 0     & 0     & 0     & 7000  & 7000  \\
$T_{\mathrm{warm}}$  & 0     & 0     & 0     & 0     & 5000  & 5000    \\
$\beta_{\max}$       & 0     & 0     & 0     & 0     & 0.1   & 0.1   \\
Rollout steps        & 64    & 64    & 64    & 256   & 256   & 256   \\
Batch size           & 16384 & 16384 & 16384 & 32768 & 32768 & 32768 \\
\addlinespace
\multicolumn{7}{@{}l}{\textit{Shared across all environments}}\\
Parallel envs        & \multicolumn{6}{c}{4096} \\
Discount $\gamma$    & \multicolumn{6}{c}{0.999} \\
GAE $\lambda$        & \multicolumn{6}{c}{0.95} \\
Clip $\epsilon$      & \multicolumn{6}{c}{0.2} \\
Entropy coef.        & \multicolumn{6}{c}{0.1} \\
Epochs per update    & \multicolumn{6}{c}{2} \\
Learning rate        & \multicolumn{6}{c}{$5 \times 10^{-4}$} \\
Optimizer            & \multicolumn{6}{c}{Adam} \\
Hidden layers        & \multicolumn{6}{c}{2} \\
Hidden dim           & \multicolumn{6}{c}{256} \\
Activation           & \multicolumn{6}{c}{Leaky ReLU} \\
\bottomrule
\end{tabular}
\end{table}

All policy and value networks are MLPs with the architecture given in Table~\ref{tab:hyperparams}; the policy outputs a Gaussian with a state-dependent learned standard deviation.

\subsection{Computing Resources}\label{app:compute}
All experiments and model training described in this work were conducted on a single compute node equipped with two NVIDIA H100 Tensor Core GPUs and 1.0 TiB of system memory.

\subsection{Environments}\label{app:environments}

For the experiments involving pure CIP maximization we evaluate on four systems of increasing dimensionality: cart pole, double pendulum, triple pendulum, and gibbon, a planar humanoid-like system. These are the benchmarks of \citet{shah2026controllable}, with modifications: we add joint damping to the triple pendulum ($3.5$) and gibbon ($5.0$). Without damping, the policy discovers trivial maximizers of CIP, spinning the free joints to accumulate unbounded velocities rather than seeking unstable configurations. Damping bounds the attainable velocities and thereby removes these degenerate optima. We also reduce the gear strength on the original gibbon environment to $50.0$.

For experiments regarding locomotion: hopper and walker, we use the default environments from the DeepMind control suite \citep{Tunyasuvunakool_2020}. For both we set the internal joint damping to $1.0$. Without damping we observed uncontrolled thrashing behavior as a maximizer of the CIP objective.

\subsection{Training and Evaluation}\label{app:training}

All baseline methods are trained under identical environmental conditions: the same environments, gear and damping settings, and the same reset distribution ($x_0 \sim \rho_0(\cdot)$). They differ only in the intrinsic objective. Episodes reset to the hanging configuration in cart pole, double pendulum, and triple pendulum, so that the intrinsic reward alone must drive the exploration from the passive state to the upright one, without the agent ever being seeded there. In the gibbon environment, resets are split evenly between the hanging configuration and the vicinity of the upright stand. For hopper and walker the environments are dropped from a height above the ground with a small perturbation noise applied to its initial position.

\subsection{Baselines}


We compare against four popular intrinsic objectives: DIAYN \citep{eysenbach2018diversityneedlearningskills}, SMM \citep{lee2020efficientexplorationstatemarginal}, ICM \citep{pathak2017curiosity} and APT \citep{liu2021behaviorvoidunsupervisedactive}. Skill-based methods (DIAYN and SMM) are evaluated on the skill achieving the highest average height in Table~\ref{tab:results} and Table~\ref{tab:reward-simplification}.

\end{document}